\documentclass[sigconf]{acmart}

\usepackage{booktabs} 

\setcopyright{acmcopyright}

\acmConference[SAC'22]{ACM SAC Conference}{April 25 –April 29, 2022}{Brno, Czech Republic}
\acmYear{2022}
\copyrightyear{2022}

\usepackage{mathtools}
\usepackage{bbm}
\usepackage{bm}
\usepackage{amsthm}
\usepackage{tabulary}
\usepackage{adjustbox}
\usepackage{enumitem}
\usepackage{verbatim}
\usepackage[utf8x]{inputenc}

\newtheorem{theorem}{Theorem}
\newcommand{\card}[1]{\lvert#1\rvert}

\DeclareMathOperator*{\argmax}{arg\,max}
\DeclareMathOperator*{\E}{\mathbb{E}}

\newcommand{\defeq}{\vcentcolon=}

\begin{document}
\title{Supervised Graph Contrastive Pretraining for Text Classification}
\titlenote{A condensed version of this paper has been accepted to ACM SAC'22.}

\renewcommand{\shorttitle}{GCPT}

\author{Samujjwal Ghosh}
\orcid{0000-0003-2859-7358}
\affiliation{%
 \institution{IIT Hyderabad}
}
\email{samujjwal86@gmail.com}

\author{Subhadeep Maji}
\affiliation{%
 \institution{Amazon}
}
\email{msubhade@amazon.com}

\author{Maunendra Sankar Desarkar}
\affiliation{%
 \institution{IIT Hyderabad}
}
\email{maunendra@cse.iith.ac.in}
%
%
%
%

\renewcommand{\shortauthors}{S Ghosh et al.}
\newcommand{\comr}[1]{\textcolor{red}{(Comment: #1)}}
\newcommand{\comb}[1]{\textcolor{blue}{#1}}

\begin{abstract}

    Contrastive pretraining techniques for text classification has been largely studied in an unsupervised setting. However, oftentimes labeled data from related tasks which share label semantics with current task is available. We hypothesize that using this labeled data effectively can lead to better generalization on current task. In this paper, we propose a novel way to effectively utilize labeled data from related tasks with a graph based supervised contrastive learning approach. We formulate a token-graph by extrapolating the supervised information from examples to tokens. Our formulation results in an embedding space where tokens with high/low probability of belonging to same class are near/further-away from one another. We also develop detailed theoretical insights which serve as a motivation for our method. In our experiments with $13$ datasets, we show our method outperforms pretraining schemes by $2.5\%$ and also example-level contrastive learning based formulation by $1.8\%$ on average. In addition, we show cross-domain effectiveness of our method in a zero-shot setting by $3.91\%$ on average. Lastly, we also demonstrate our method can be used as a noisy teacher in a knowledge distillation setting to significantly improve performance of transformer based models in low labeled data regime by $4.57\%$ on average.
\end{abstract}

%
%

\ccsdesc[300]{Human-centered computing~Social network analysis}
\ccsdesc[500]{Computing methodologies~Information extraction}
\ccsdesc[300]{Computing methodologies~Supervised learning}
\ccsdesc[100]{Computing methodologies~Neural networks}
\ccsdesc[500]{Computing methodologies~Learning latent representations}
\keywords{Contrastive Learning, Graph Neural Network, Text Classification, Text Representation, Zero-shot Classification, Disaster Response, Sentiment Analysis}

\maketitle

\section{Introduction}
Lack of labeled data is probably the most pervasive problem limiting the effectiveness of modern deep learning methods. In text classification, the widespread approach to get around this has been the classic pretrain-finetune paradigm where a model is learnt on a large corpus and then fine-tuned using a domain specific dataset. This pretrain-finetune paradigm is able to perform significantly better than just finetuning the model. Pretraining is generally performed using a large amount of unlabeled data. There are many situations where labeled data from \textit{related} datasets are available and can be utilized for pretraining. Our definition of \textit{related} is fairly broad; we consider two datasets to be related if their label space has the same semantics e.g. tweets related v/s unrelated to a natural disaster, positive v/s negative sentiment in a review, etc.

Due to the recent success of contrastive learning in improving generalization and theoretical guarantees in both Vision \cite{khosla2020supervised, chen2020simple, wang2021dense} and NLP~\cite{Wu2020CLEARCL,gunel2021supervised, Liao2021SentenceEU} domains, we explore supervised contrative pretraining in the context of binary text classification. We propose a graph contrastive pretraining technique to effectively utilize labeled data from related datasets when attempting to build a robust classifier for the current dataset. It learns an embedding space for unique tokens in the vocabulary by constructing a token-graph where the collocation of tokens defines the neighborhood.

Existing contrastive learning approaches fall on either unsupervised pretraining~\cite{Wu2020CLEARCL, Giorgi2021DeCLUTRDC} or supervised fine-tuning~\cite{huang2021tokenlevel, gunel2021supervised, sup_icassp}.
Unsupervised contrastive pretraining based techniques do not utilize supervised information even when available in related datasets. On the other hand, supervised fine-tuning based approaches uses contrastive loss as an additional objective only in addition to cross-entropy loss. We propose supervised contrastive pretraining scheme which utilizes supervised information from related datasets with contrastive formulation. We extrapolate the label information on examples to tokens as their conditional probability $\text{Pr}(\text{label}\mid \text{token})$ and use this label information in a supervised contrastive learning formulation to learn a token embedding space. We contrast between the highly confident tokens within a class with similar tokens from other class. Our contrastive formulation results in an embedding space where tokens belonging to different classes are well separated, a desirable characteristic of feature space in a classification setting. We refer our proposed technique as \textbf{G}raph \textbf{C}ontrastive \textbf{P}re-training for \textbf{T}ext (GCPT). We also develop theoretical guarantees around effectiveness of GCPT by building upon~\cite{saunshi2019theoretical}.

We conduct several experiments across $5$ different sentiment classification and $8$ disaster management datasets in binary classification setting to gain a robust understanding of our method's effectiveness. We find that our method outperforms word embedding techniques trained on unlabeled domain-specific datasets (e.g. disaster related tweets) by $2.5\%$ on average. We also show that our method is an effective pretraining scheme by combining it with several different baseline models from disaster and sentiment classification domain; in each case it improves the performance of underlying model, on average, by $1.87\%$. We also compare against example-level contrastive learning formulation and show our token level contrastive formulation is better by $1.8\%$ on average. Our method also achieves robust zero-shot performance wherein no labeled data from current dataset is available (a very realistic scenario in disaster management). Once again it outperforms several baselines by $3.91\%$ on average. Finally, we also improve upon BERT in a knowledge distillation setup wherein our model works as a noisy teacher for a BERT student. In this case, we improve upon BERT's performance by $4.57\%$ on average.
\section{Related Work} \label{section:related}
Most existing works attempting to utilize contrastive learning for text classification can be categorized into two paradigms  a) unsupervised pretraining or b) supervised fine-tuning.
Our work attempts to combine both approaches and is motivated by recent success of contrastive pretraining for language modelling and sentence representations leading to good generalizability across wide range of applications~\cite{gunel2021supervised,Liao2021SentenceEU,Wu2020CLEARCL,huang2021tokenlevel}.

\textbf{Unsupervised Pretraining:} Pretraining based approaches has been largely explored in an unsupervised setting only. However, oftentimes supervised corpora is available for other \textit{related} datasets, e.g. labeled sentiment analysis dataset for ``Book Reviews'' when attempting to learn a sentiment analysis classifier for ``Kitchen Items''.
\cite{Giorgi2021DeCLUTRDC,fang2020cert} propose self-supervised contrastive learning methods for language models using data augmentation techniques such as back-translation and sampling spans from same or different documents as positives and negatives respectively. \cite{Wu2020CLEARCL} extends the idea to learn sentence level representations by proposing a number of different sentence level augmentation strategies. However, unsupervised or self-supervised contrastive learning based methods have a natural difficulty when it comes to NLP, this is because unlike images changing a few words or even their ordering in a sentence can significantly alter the semantics.

\textbf{Supervised Fine-tuning:} On the other hand,
\cite{gunel2021supervised,huang2021tokenlevel,Liao2021SentenceEU, sup_icassp} takes a different approach by using the labeled information to learn a feature space where examples belonging to different classes are further away than examples from same class. They achieve this by typically adding a contrastive loss component to the standard cross-entropy loss. These methods have shown to be robust to label noise and significantly improve upon large language models such as BERT.

Our work combines above mentioned approaches by utilizing \emph{supervised} information during \emph{pretraining} using contrastive learning. However, there are two key differences in our formulation. First, we attempt to learn a token-level (e.g. word) embedding space while most existing methods are example-level (e.g. sentences). In our experiments, we show our token level formulation is better in comparison to an example level contrastive learning formulation. Second, our method is naturally suited to a cross-domain setting because of our token level pretraining setup. This allows us to use labeled data from related datasets effectively, while existing supervised contrastive learning methods are mostly restricted to fine-tuning the contrastive loss jointly with cross-entropy loss on the current labeled dataset alone. To establish cross-domain applicability of our method, we show zero-shot text classification performance i.e no labeled data being available for the current dataset.
\section{Approach} \label{section:approach}
At a very high level, our approach is a pretraining technique for learning a token embedding space with the end objective of improving text classification performance on datasets with limited labeled data.
We assume a shared label space between the previous and current datasets. This is reasonable assumption as the type of information to collect generally stays consistent across disasters.
Using labeled pretraining corpora, from previous datasets, in the form of $\{(e,y)\}$ where $e$ is an example, and $y$ is it's class label, our method contrasts between representative candidate tokens from each class.
Candidates are found by extrapolating the label information available at example-level to tokens as the conditional probability $\text{Pr}(\text{label} \mid \text{token})$.
%
This conditional label information is then used in a contrastive formulation to learn an embedding space where tokens from different classes are well separated. We also present novel theoretical results building upon~\cite{saunshi2019theoretical}, the theoretical analysis serves as motivation for our algorithm.

We refer to the current dataset as $\mathbf{O}$ and the union of related datasets as $\mathbf{P}$. The elements in $\mathbf{O}$ (and $\mathbf{P}$) are pairs $(\mathbf{x}, y)$ where $\mathbf{x}$ is an example (e.g tweet) and $y\in\mathcal{Y}$ is the class label, $\card{\mathcal{Y}}=C$ is the number of classes. We assume label space $\mathcal{Y}$ to be the same across current and related datasets. Also, $|\mathbf{O}|$ is much smaller than $|\mathbf{P}|$, indicative of our limited labeled data setup.
\subsection{Token Label Conditional} \label{ssec:token_extrapolation}
To extrapolate the label information from examples to tokens, we introduce a token graph $\mathbf{G}=(V,E)$. Specifically, set of all unique tokens with minimum frequency ($\geq 5$) from $\mathbf{p} \in \mathbf{P}$ are vertices ($V$) in the graph. Edges ($E$) are defined by pointwise mutual information (PMI) between two tokens over all $\mathbf{p} \in \mathbf{P}$. Given a label $c$, let $\mathbf{P}^c \subset V$ be a set of tokens defined with,
\begin{align} \label{eq:theta_estimate}
	\theta^{j}_{x} & =  \frac{n^{j}(x)}{\sum\limits_{c=1}^{C} n^{c}(x)};
	~~~~\text{s.t.} \quad \theta^{c}_{x} & \geq \tau \geq \frac{1}{C}; ~~~~
	c & = \argmax_{j \in \{1,\cdots,C\}} \theta^{j}_{x}
\end{align}
where $n^{c}(x)$ is the frequency of token $x$ in examples $\mathbf{p} \in \mathbf{P}$ such that their label is $c$. Intuitively, $x\in \mathbf{P}^c$ has the highest conditional probability of belonging to label $c$ (ties broken arbitrarily) amongst all labels.
%
\begin{figure*}
\includegraphics[width = 0.999\linewidth]{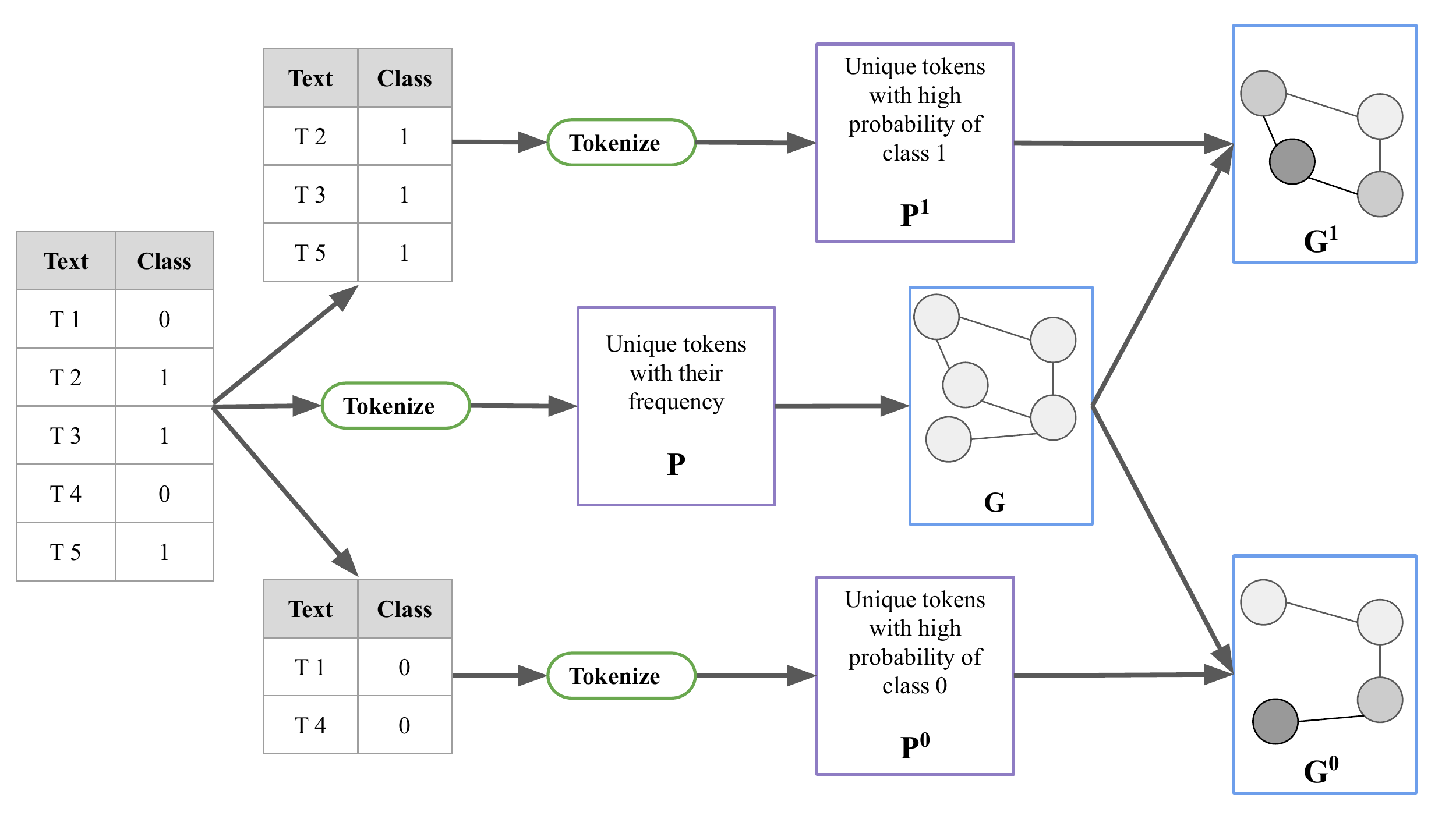}\\
\caption{High level illustration of our token-level conditional and graph construction in binary setting.}
\label{fig:dataflow}
\end{figure*}
The normalized frequency $\theta^{c}_{x}$ determines the confidence of label-token association, and only tokens with high confidence (i.e $\geq \tau$) have a conditional label. 
Considering high confidence within a class ensures exclusion of insignificant tokens (e.g. stopwords) as candidates as they occurs uniformly across classes. Note that these tokens may be present in the graph but not considered as candidate. On the other hand, to avoid very less frequent tokens as candidates, we only consider tokens with a minimum frequency when constructing the graph $\mathbf{G}$. 
Using sets $\mathbf{P}^c$ we define the induced subgraphs for individual labels $c \in \mathcal{Y}$ as $\mathbf{G}^c=(\mathbf{P}^c,E^{c},W)$.
Note that sets $\mathbf{P}^c$ are mutually exclusive, and not all tokens $x \in V$ belong to $\mathbf{P}^c$ (there can be only one $c$ for a particular token $x$).
To learn a token (node) representation for $x \in V$ we apply a $2$-layer GCN~\citep{gcn} on $\mathbf{G}$.
Specifically,
\begin{align} \label{eq:gcn_node_rep}
	H^{(l+1)} &= \sigma \left( D^{-\frac{1}{2}} \hat{A} D^{-\frac{1}{2}} H^{(l)}W^{(l)}\right)
\end{align}
with $H^{(0)} = \text{Glove}(\mathbf{.})$. Here, $D$ is the diagonal degree matrix of $\mathbf{G}$ and $\hat{A} = A + I$, where $A$ is the weighted adjacency matrix of $\mathbf{G}$ and $l$ is the layer with values $l\in\{0,1\}$ in this case.
GCN based node representations capture co-occurrence patterns between tokens from $\mathbf{P}$ because of the underlying message-passing over neighborhoods. In addition, because $\mathbf{G}$ can be seen as a union of induced subgraphs $\mathbf{G}^c$, the message-passing with GCN layers over $\mathbf{G}$ smoothens the node representations across these subgraphs and acts as a form of regularization to our contrastive formulation. This regularization is helpful because our extrapolation of label information to tokens is based on observed $\mathbf{P}$ and therefore empirical.
%
%
\subsection{Graph Contrastive Pretraining} \label{ssec:contrastive_pretrain}
Let $\cup_{c=1}^{C} \mathbf{G}^c$ be the union of the induced subgraphs for individual labels $c\in\mathbf{Y}$. For node (token) $n\in \mathbf{G}^c$ ensuring $n$ is closer to its neighbors from $\mathbf{G}^{c}$ and further away from nodes in $\mathbf{G}^{\lnot c}$ would result in tokens belonging to the same (or different) class to be close (or faraway) from one other in the embedding space. To this effect, we propose a graph contrastive learning method over nodes in $\cup_{c=1}^{C} \mathbf{G}^c$. Specifically, for a candidate token $n\in\mathbf{G}^c$, we sample its positive and negative neighbors as follows:
\begin{itemize}
	\item \textbf{Positive}: The positives are sampled from neighborhood of node $n$ in the induced subgraph $\mathbf{G}^c$. We draw $S$ positive samples with replacement from this neighborhood (Nbh) i.e $M^{+} \sim_{iid} \text{Nbh}(n)$, with $|M^{+}|=S$.
	\item \textbf{Negative}: The negatives are sampled uniformly at random from induced subgraphs other than $\mathbf{G}^c$ i.e $\mathbf{G}^{\lnot c}$. We draw $S$ negatives samples with replacement from $\mathbf{G}^{\lnot c}$ i.e $M^{-} \sim_{iid} \mathcal{U}(P^{\lnot c})$, with $|M^{-}|=S$. $\mathcal{U}$ represents uniform distribution.
\end{itemize}

Figure~\ref{fig:dataflow} provides an illustration of our sampling process on $\mathbf{G}^c$ and $\mathbf{G}^{\lnot c}$.
Using candidate token $n$, the positive samples $M^{+}=m^{+}_{1..s}$  and negative samples $M^{-}=m^{-}_{1..s}$ we define a contrastive loss formulation with equation~\eqref{eq:contrastive_loss}. The proposed loss is a variant of~\citep{supervised-contrast} where positive samples are drawn from graph neighborhoods. Formally, the loss on node $n$ is defined as $l_n$ as follows,
\begin{align}  \label{eq:contrastive_loss}
	& f(m,n)  = \text{exp}\left ( \frac{z_m \cdot z_n}{\gamma}\right) \\ \nonumber
	& l_n  = -\frac{1}{S} \sum_{m^{+} \in M^+}
	\log \left( \frac{f(m^{+}, n)} {\sum\limits_{m \in M^+ \cup M^{-}} f(m, n)} \right)
\end{align}
Where $z_n = H^{(2)}(n)$ is the GCN representation of node $n$ and $\gamma$ is a temperature parameter for softmax. Intuitively, the loss tries to bring the representation of nodes from $\mathbf{G}^c$ close to each other in comparison to nodes from $\mathbf{G}^{\lnot c}$. This formulation results in greater feature separation of different classes in the embedding space. The overall loss is obtained by drawing a random sample of $N$ such candidate tokens by first sampling a class and then sampling a node from the set of nodes belonging to that class.
\begin{align} \label{eq:overall_loss}
	\mathcal{L} & = \frac{1}{N} \left( \sum_{c \sim \mathcal{U}(1..C)} \sum_{n \sim \mathcal{U}(P^c)} l_n \right)
\end{align}
While equation~\eqref{eq:overall_loss} is defined only on nodes in $\cup_{c} \mathbf{G}^c$, representation for nodes in $\mathbf{G}\setminus \cup_{c} \mathbf{G}^c$ also get modified during pretraining because GCN layers are trained end-to-end with the contrastive loss $\mathcal{L}$. This results in better overall separation in the embedding space for all tokens in the vocabulary.

The success of the unsupervised contrastive learning methods rests on the reliability of positive (and negative) samples belonging to the same (and different) class as the candidate. This is a difficulty with many existing unsupervised contrastive learning methods~\cite{saunshi2019theoretical}. Although our contrastive formulation is not unsupervised, it relies on extrapolation of example level supervision to tokens and potentially suffers from the label noise. We address this issue by applying our contrastive formulation to tokens with high confidence label association (i.e. $\cup_{c=1}^{C} \mathbf{G}^c$), resulting in better guarantees on positive (and negative) samples belonging to the same (and different) class as the candidate. In \S~\ref{ssec:theory}, we build upon the theoretical analysis in~\citep{saunshi2019theoretical} to develop deeper insights on token level extrapolation of supervision for contrastive losses.
%
\subsection{Theoretical Insights} \label{ssec:theory}
Our theoretical results are for binary classification problems, i.e $C=2$, also we restrict ourselves to only single positive and negative sample for each candidate $x$ (i.e $S=1$). The extension of these results for multi-class problems with blocks of $S$ positive and negative samples is left as future work. In this setting, the contrastive loss in equation~\eqref{eq:contrastive_loss} is,
\begin{align*}
	l_n & = -\log\left(
	\frac{\exp(\frac{z_{m^{+}} \cdot z_n}{\gamma})}
	{\exp(\frac{z_{m^{+}} \cdot z_n}{\gamma}) + \exp(\frac{z_{m^{-}} \cdot z_n}{\gamma})} \right) \\ \nonumber
	& = \log \left(1 + \exp \left (\frac{z_n \cdot (z_{m^{-}} - z_{m^{+}})}{\gamma}\right ) \right ) \\ \nonumber
	& \defeq h(z_n \cdot (z_{m^{-}} - z_{m^{+}}))
\end{align*}
For ease of exposition, with some abuse of notation we will use $h(z_n \cdot (z_{m^{-}} - z_{m^{+}}))$ and $h(z_n, z_{m^{+}}, z_{m^{-}})$ interchangeably.

Our theoretical analysis is not restricted to the exact sampling scheme discussed in~\S~\ref{ssec:contrastive_pretrain} and is more broadly applicable. The only assumption is the label conditional on every token $x$ is modeled as Bernoulli distribution parameterized by $\theta_{x}$. Note that, we do not make further assumption on estimation of $\theta_x$ (such as equation~\eqref{eq:theta_estimate}).

We now describe a generative process from which our contrastive formulation follows. First sample a token $x$ from data distribution $p_{\text{data}}$, then sample positive and negative classes $c^{+}\sim \text{Ber}(\theta_x), c^{-} \sim \text{Ber}(1-\theta_{x})$ according to Bernoulli distributions dependent on the conditional probability for token $x$. Then sample positive and negative examples for contrastive learning w.r.t token $x$ from distribution $D_{c^{+}}$ and $D_{c^{-}}$. Note that the sampling scheme discussed in~\S~\ref{ssec:contrastive_pretrain} is a special case of the above because we don't sample $c^{+},c^{-}$ explicitly. Instead, every token already has the most likely label associated with it (in accordance with equation~\eqref{eq:theta_estimate}). The overall contrastive loss can then be expressed as,
\begin{align} \label{eq:unsupervised_loss}
	L(z) =  \E_{\substack{x \sim p_{\text{data}} \\
			c^{+} \sim  \text{Ber}(\theta_{x}) \\ c^{-} \sim \text{Ber}(1-\theta_{x})}}
	\left[ \E_{\substack{m^{+} \sim D_{c^{+} } \\ m^{-} \sim D_{c^{-} }}} \big[ h(z_{x}, z_{m^{+}}, z_{m^{-}}) \big] \right]
\end{align}
In equation~\eqref{eq:unsupervised_loss}, $D_{c^{+}}, D_{c^{-}}$ are data distributions of classes $c^{+}, c^{-}$ respectively. For ease of exposition, we will drop the $\text{Ber}(.)$ notation and instead use $ c^{+} \sim  \theta_{x}, c^{-} \sim (1-\theta_{x})$. Following~\citep{saunshi2019theoretical}, we split $L(z)$ into two components  based on $c^{+} = c^{-}$, the ``class-collision loss'' (CCL) and ``ideal contrastive loss'' (ICL).
%
		\begin{align*}
			L^{\text{CCL}}(z) & =  \E_{\substack{x \sim p_{\text{data}} \\
					c^{+} \sim  \theta_{x} \\ c^{-} \sim (1-\theta_{x})}}
			\left[ \E_{\substack{m^{+} \sim D_{c^{+} } \\ m^{-} \sim D_{c^{-} }}} \big[ h(z_{x}, z_{m^{+}}, z_{m^{-}})
			\mid c^{+} = c^{-} \big]  \right] \\ \nonumber
			L^{\text{ICL}}(z) & = \E_{\substack{x \sim p_{\text{data}} \\
					c^{+} \sim  \theta_{x} \\ c^{-} \sim (1-\theta_{x})}}
			\left[ \E_{\substack{m^{+} \sim D_{c^{+} } \\ m^{-} \sim D_{c^{-} }}} \big[ h(z_{x}, z_{m^{+}}, z_{m^{-}})
			\mid c^{+} \neq c^{-} \big] \right]
		\end{align*}
%
The ideal behaviour of a learning algorithm minimizing contrastive loss $L(z)$ should be to minimize both CCL ($L^{\text{CCL}}(z)$) as well as ICL ($L^{\text{ICL}}(z)$) loss components. However, the behaviour of the two loss components is very different from one another. The CCL component comes from the penalty a learning algorithm pays from class specific positive and negative sampling errors while ICL component helps in learning contrastive feature representations $z$. The following result allows us to characterize the relation of label-token association parameters; $\theta_x$ with CCL and ICL components of $L(z)$.
\begin{theorem} \label{thm:upper_bound}
	The contrastive loss $L(z)$ is upper bounded by $c L^{\text{CCL}}(z) + L^{\text{ICL}}(z)$, where $c=2\tau(1-\tau) \in[0.5,1]$ and $\tau \geq \frac{1}{2}$.
\end{theorem}
\begin{proof}
For ease of exposition, we will denote $h(z_{x}, z_{m^{+}}, z_{m^{-}})$ as $\mathbbm{g}$. We first separate $L^{\neq}(z)$ from $L(z)$ as follows,
\begin{align*}
	L(z) & =  \E_{\substack{x \sim p_{\text{data}} \\ \nonumber
			c^{+} \sim  \theta_{x} \\ c^{-} \sim (1-\theta_{x})}}
	\left[ \E_{\substack{m^{+} \sim D_{c^{+} } \\ m^{-} \sim D_{c^{-} }}} \left[ \mathbbm{g} \right] \right]  \\
	& =^{(\textbf{a})} \E_{\substack{x \sim p_{\text{data}}  \\
			c^{+} \sim  \theta_{x} \\ c^{-} \sim (1-\theta_{x})}}
	\left[
	\E_{\mathbbm{1}(c^{+} = c^{-})} \left[
	\E_{\substack{m^{+} \sim D_{c^{+} } \\ m^{-} \sim D_{c^{-} }}} \left[ \mathbbm{g} \mid \mathbbm{1}(c^{+} = c^{-}) \right]
	\right]
	\right] \\ \nonumber
	& =^{(\textbf{b})} \E_{\substack{x \sim p_{\text{data}} \\
			c^{+} \sim  \theta_{x} \\ c^{-} \sim (1-\theta_{x})}}
	\left[ \substack{\E_{\substack{m^{+} \sim D_{c^{+} } \\ m^{-} \sim D_{c^{-} }}} \left[ \mathbbm{g} \mid  c^{+} = c^{-}  \right] \text{Pr}(\mathbbm{1}(c^{+} = c^{-})) \\ \bm{+} \\
		\E_{\substack{m^{+} \sim D_{c^{+} } \\ m^{-} \sim D_{c^{-} }}} \left[ \mathbbm{g} \mid  c^{+} \neq c^{-}  \right] \text{Pr}(\mathbbm{1}(c^{+} \neq c^{-}))} \right]  \\ \nonumber
	& =^{(\textbf{c})} \E_{\substack{x \sim p_{\text{data}} \\
			c^{+} \sim  \theta_{x} \\ c^{-} \sim (1-\theta_{x})}}
	\left[ \substack{\E_{\substack{m^{+} \sim D_{c^{+} } \\ m^{-} \sim D_{c^{-} }}} \big[ \mathbbm{g} \mid  c^{+} = c^{-}  \big] (2\theta_x(1-\theta_x)) \\ \bm{+} \\
		\E_{\substack{m^{+} \sim D_{c^{+} } \\ m^{-} \sim D_{c^{-} }}} \big[ \mathbbm{g} \mid  c^{+} \neq c^{-}  \big]
		(1-2\theta_x(1-\theta_x))} \right]  \\ \nonumber
	& \leq^{(\textbf{d})} \E_{\substack{x \sim p_{\text{data}} \\
			c^{+} \sim  \theta_{x} \\ c^{-} \sim (1-\theta_{x})}}
	\left[ \substack{\E_{\substack{m^{+} \sim D_{c^{+} } \\ m^{-} \sim D_{c^{-} }}} \left[ \mathbbm{g} \mid  c^{+} = c^{-}  \right] 2\theta_x(1-\theta_x)\\ \bm{+} \\
		\E_{\substack{m^{+} \sim D_{c^{+} } \\ m^{-} \sim D_{c^{-} }}} \left[ \mathbbm{g} \mid  c^{+} \neq c^{-}  \right]} \right]  \\ \nonumber
	& =^{(\textbf{e})} \E_{\substack{x \sim p_{\text{data}} \\
			c^{+} \sim  \theta_{x} \\ c^{-} \sim (1-\theta_{x})}}
	\left[\E_{\substack{m^{+} \sim D_{c^{+} } \\ m^{-} \sim D_{c^{-} }}} \left[ \mathbbm{g} \mid  c^{+} = c^{-}  \right] 2\theta_x(1-\theta_x) \right] \bm{+} L^{\text{ICL}}(z) \nonumber
\end{align*}
Where, step (a) follows from Iterated expectation w.r.t $\mathbbm{1}(c^{+} = c^{-})$. Step (b) follows from opening up the expectation w.r.t $\mathbbm{1}(c^{+} = c^{-})$. Step (c) follows from property of Bernoulli distribution, i.e $\text{Pr}(\mathbbm{1}(c^{+} = c^{-})) = 2\theta_x (1-\theta_x)$. Step (d) follows from non-negativity of $2\theta_x (1-\theta_x)$.

Next, we connect true parameters $\theta_x$ with the estimates from equation~\eqref{eq:theta_estimate}, which allows us to get a upper bound on the first term of (e). First we note that equation~\eqref{eq:theta_estimate} is a MLE of Bernoulli parameter $\theta_x$. From a classical result on consistency of MLEs, we note that for any continuous function $f$ of $\theta_x$; $f(\hat{\theta_x})$ is a consistent estimator of $f(\theta_x)$. Formally,
\begin{align} \label{eq:mle_cons}
	\lim\limits_{n \to \infty} P_{\theta_{x}}(\mid f(\hat{\theta_x}) - f(\theta_x) \mid  \geq \epsilon) = 0
\end{align}
(see~\citep{casella2021statistical} for a proof). Noting that $f(\theta_x) = 2\theta_x(1-\theta_x)$ is a continuous function of $\theta_x$, we apply result~\eqref{eq:mle_cons} to $f(\theta_x)$ and replace $f(\theta_x)$ with its estimate $f(\hat{\theta_x})$ in step (e) above,
		\begin{align*}
			L(z) & \leq \E_{\substack{x \sim p_{\text{data}} \\
					c^{+} \sim  \theta_{x} \\ c^{-} \sim (1-\theta_{x})}}
			[\E_{\substack{m^{+} \sim D_{c^{+} } \\ m^{-} \sim D_{c^{-} }}} \big[ \mathbbm{g} \mid  c^{+} = c^{-}  \big]
			2 \hat{\theta_x}(1-\hat{\theta_x})] \bm{+} L^{\text{ICL}}(z)  \\
			& \leq 2\tau(1-\tau) L^{\text{CCL}}(z) + L^{\text{ICL}}(z)
		\end{align*}
We note that $2\hat{\theta_x}(1-\hat{\theta_x})$ is a monotonically decreasing function of $\hat{\theta_x}$ in domain $[0.5,1]$ (we chose $\hat{\theta_x} \geq 0.5$ for all $x$ from equation~\eqref{eq:theta_estimate}). This implies there exists $\tau \geq 0.5$ such that $2\hat{\theta_x}(1-\hat{\theta_x}) \leq 2\tau(1-\tau)$ for all $x$, from which the claim follows immediately.
\end{proof}
%
As an immediate consequence of Theorem~\ref{thm:upper_bound}, by setting a $\tau > 0.5$ our sampling algorithm can be conceptualized as minimizing the above upper bound on $L(z)$ w.r.t $\tau$ by down weighing the CCL component. Of course, if we set $\tau=1$, we can nullify the entire CCL component. However this comes at the cost of having very few candidates $x$, which satisfy $\theta_x=1$ in any practical setting. Thus severely limiting the size of the training data available for minimizing the empirical contrastive loss in equation~\eqref{eq:overall_loss}.

Theorem~\ref{thm:upper_bound} although providing an upper bound on $L(z)$ presents only a partial picture because the upper bound can be loose and minimizing it w.r.t $\tau$ may not explain the effects on minimization of $L(z)$. Therefore, our next step is to prove the existence of a suitable $\tau$ dependent lower bound on $L(z)$.
\begin{theorem} \label{thm:lower_bound}
	The contrastive loss $L(z)$ is lower bounded by
	\begin{align*}
		(1-2\tau(1-\tau)) \E_{\substack{x \sim p_{\text{data}} \\
				c^{+} \sim  \theta_{x} \\ c^{-} \sim (1-\theta_{x})}} \left[ h(z_{x}\cdot(\mu_{c_{x}^{-}} - \mu_{c_{x}^{+}})) \mid c^{+} \neq c^{-}  \right]
	\end{align*}
	where $\mu_{c_{x}^{+}}=\E_{p\in D_{c^{+}}}[p]$, $\mu_{c_{x}^{-}}=\E_{p\in D_{c^{-}}} [p]$ are the mean of the two classes $c^+,c^{-}$. Where $2\tau(1-\tau) \in[0.5,1]$ and $\tau \geq \frac{1}{2}$.
\end{theorem}
\begin{proof}
	Using the definition of $L(z)$ and following steps (a)-(c) in proof of Theorem~\ref{thm:upper_bound} we arrive at,
			\begin{align*}
				L(z) & = \E_{\substack{x \sim p_{\text{data}} \\
						c^{+} \sim  \theta_{x} \\ c^{-} \sim (1-\theta_{x})}}
				\left[ \substack{\E_{\substack{m^{+} \sim D_{c^{+} } \\ m^{-} \sim D_{c^{-} }}} [ h(z_{x}\cdot(m^{-} - m^{+})) \mid  c^{+} = c^{-}] (2\theta_x(1-\theta_x)) \\ \bm{+} \\
					\E_{\substack{m^{+} \sim D_{c^{+} } \\ m^{-} \sim D_{c^{-} }}} [ h(z_{x}\cdot(m^{-} - m^{+})) \mid  c^{+} \neq c^{-}] (1-2\theta_x(1-\theta_x))} \right]  \\
				& \geq^{{(\textbf{b})}} \E_{\substack{x \sim p_{\text{data}} \\
						c^{+} \sim  \theta_{x} \\ c^{-} \sim (1-\theta_{x})}}
				\left[\E_{\substack{m^{+} \sim D_{c^{+} } \\ m^{-} \sim D_{c^{-} }}} \substack{[ h(z_{x}\cdot(m^{-} - m^{+})) \mid  c^{+} \neq c^{-}] (1-2\theta_x(1-\theta_x))} \right] \\\
				& \geq^{{(\textbf{c})}} \E_{\substack{x \sim p_{\text{data}} \\
						c^{+} \sim  \theta_{x} \\ c^{-} \sim (1-\theta_{x})}}
				\left[ (1-2\theta_x(1-\theta_x)) h(z_{x}\cdot(\mu_{c_{x}^{-}} - \mu_{c_{x}^{+}})) \mid c^{+} \neq c^{-} \right]
			\end{align*}
	where step (b) follows from non-negativity of expectation of loss functions and the fact that $2\theta_x(1-\theta_x) \geq 0$. Step (c) follows from convexity of log-sum-exp function $h$ (w.r.t $z$) and Jensen's inequality applied to inner expectation and $h$. Once again as in proof of Theorem~\ref{thm:upper_bound} using consistency of MLE (equation~\eqref{eq:mle_cons}), we replace $(1-2\theta_x(1-\theta_x))$ with its estimated value in step (c) above.
			\begin{align*}
				L(z) & \geq \E_{\substack{x \sim p_{\text{data}} \\
						c^{+} \sim  \theta_{x} \\ c^{-} \sim (1-\theta_{x})}}
				\left[ (1-2\hat{\theta_x}(1-\hat{\theta_x})) h(z_{x}\cdot(\mu_{c_{x}^{-}} - \mu_{c_{x}^{+}})) \mid c^{+} \neq c^{-}  \right]
				\\
				& \geq (1-2\tau(1-\tau))  \E_{\substack{x \sim p_{\text{data}} \\
						c^{+} \sim  \theta_{x} \\ c^{-} \sim (1-\theta_{x})}}
				\left[ h(z_{x}\cdot(\mu_{c_{x}^{-}} - \mu_{c_{x}^{+}})) \mid c^{+} \neq c^{-} \right]
			\end{align*}
	Following the proof of Theorem~\ref{thm:upper_bound}, we note that $2\hat{\theta_x}(1-\hat{\theta_x}) \leq 2\tau(1-\tau)$ for all $x$. This implies $1-2\hat{\theta_x}(1-\hat{\theta_x}) \geq 1-2\tau(1-\tau)$ for all $x$, from which the claim follows immediately.
\end{proof}
%
Therefore, attempting to maximize the lower bound in Theorem~\ref{thm:lower_bound} w.r.t $\tau$ is equivalent to minimizing $2\tau(1-\tau)$ w.r.t $\tau$. Thus, combining Theorem~\ref{thm:upper_bound} and Theorem~\ref{thm:lower_bound} we note that by choosing an appropriate value of $\tau > 0.5$ our sampling algorithm minimizes an upper bound on $L(z)$ (\S~Theorem~\ref{thm:upper_bound}) as well as maximizes a lower bound on it (\S~Theorem~\ref{thm:lower_bound}).

Once again, its important to stress that \textit{minimization} of $2\tau(1-\tau)$ in context of the analysis does not mean choosing its minima w.r.t $\tau$ (which would be $\tau=1$). Instead, the analysis simply indicates that by choosing a $\tau > 0.5$ our sampling mechanism results in a trade-off between CCL and ICL components of the contrastive loss. The correct choice of $\tau$ is still application dependent. Choosing a $\tau \approx 0.5$ results in more training data at the risk of more class-collision. On the other hand, choosing $\tau \rightarrow 1$ incurs little CCL loss at the cost of smaller training datasets for minimizing ICL and possibly hurts performance due to poor representations learnt. Thus the analysis illustrates, by choosing an appropriate threshold on label conditional probability ($\tau$) our method (in~\S~\ref{ssec:token_extrapolation} and~\S~\ref{ssec:contrastive_pretrain}) trades-off between components of contrastive loss.

\section{Experimental Setup} \label{section:experiment_setup}
We perform experiments on two collections of total $13$ datasets from disaster and sentiment domains. For each collection, we consider one dataset as the current and rest as related datasets and apply our pretraining technique on related datasets only. Training set of the current dataset is used only to train classifier models and not for contrastive pretraining. A 2-layer BiLSTM model is used as classifier for majority of our experiments as BiLSTM model is shown to have near state-of-the-art performance across the two domains on which we report our experiments. 
However, our pretraining method is generic and not tied to a particular choice of classification model. To illustrate this in our experiments we show that our pretraining mechanism improves the performance of a number of SOTA models in different application domains.

To ease the explanation of various scenarios, we follow the notations as per Table \ref{table:notations}.
\begin{table}[!hbt]
\centering
\begin{adjustbox}{
width=0.9\linewidth
}
    \begin{tabular}{@{}ll@{}}
        \toprule
        \textbf{Notation} & \textbf{Meaning} \\
        \midrule
        \textbf{$\mathbf{P}$} & Union of labeled related datasets \\
        \textbf{$\mathbf{O}_\texttt{TR}$} & Train set of current dataset \\
        \textbf{$\mathbf{O}_\texttt{TE}$} & Test set of current dataset \\
        \textbf{$\mathbf{f_{CE}(.)}$} & Training with cross-entropy loss \\
        \textbf{$\mathbf{h_{CL}^{S}(.)}$} & Pretraining with contrastive loss over sentences \\
        \textbf{$\mathbf{h_{CL}^{T}(.)}$} & Pretraining with contrastive loss over tokens \\
        \bottomrule
    \end{tabular}
\end{adjustbox}
\caption{Notations}
\label{table:notations}
\end{table}
We employ $2$ phase setup for our experiments. In the first phase, we apply a GCN with contrastive loss over the token-graph to learn vectors for each token ($\mathbf{h_{CL}^{T|S}(.)}$). Then, in the second phase, we train the classifier with Cross-Entropy (CE) loss ($\mathbf{f_{CE}(.)}$) where example tokens are represented by the pretrained token embeddings obtained in the first phase.
We report Weighted F$_1$ scores averaged over $3$ runs on $\mathbf{O}_\texttt{TE}$ with different seeds for all our experiments. We present and analyze results, aggregated across all datasets in disaster and sentiment collections, of various experiments in context of several research questions in \S~\ref{section:results}. Next, we discuss the datasets used in our experiments.
%
\subsection{Datasets} \label{ssec:datasets}
Our review sentiment collection has $5$ datasets containing reviews and the disaster collection has $8$ datasets containing tweets. The sentiment datasets are from \citep{sentiment_dataset} with binary annotation signifying positive or negative sentiment from $5$ domains (books, dvd, electronics, kitchen, video) each containing $3000$ reviews. Following \citep{sup_icassp}, we use $400$ random examples as test and split the remaining examples into $80\%:20\%$ ratio as train ($4480$) and validation ($1120$) sets. 
The disaster datasets are collected from $3$ sources. We use the $6$ datasets
``2013 Alberta Flood'' (AF), ``2013 Boston Bombing'' (BB), ``Oklahoma Tornado Season 2013'' (OT), ``Queensland Floods 2013'' (QF) and ``Sandy Hurricane 2012'' (SH) and ``2013 West Texas Explosion'' (WE)
from CrisisLexT6 \cite{crisislext6}. Additionally, also use ``2015 Nepal Earthquake'' (NE) dataset from \cite{qcri} and Kaggle's ``Natural Language Processing with Disaster Tweets'' (KL) competition data\footnote{https://www.kaggle.com/c/nlp-getting-started/overview}~\cite{kaggle_data}. Only the ``train.csv'' portion of the KL dataset is used for our experiments. All disaster datasets are tagged with binary relevance information as either \textit{relevant} (Class $1$) or \textit{not relevant} (Class $0$). We divide each dataset into train ($60$\%), validation ($10$\%) and test ($30$\%) sets. Table \ref{table:datasets} shows more details about each dataset. NE and KL datasets have special significance to our experiments. As the NE dataset falls outside of CrisisLexT6, it works as a heterogeneous dataset to verify the effectiveness of our technique. Additionally, the KL dataset contains examples from multiple disasters including storm, wildfire, flood, etc unlike the rest. We believe, incorporation of these two datasets works as an additional evidence on the effectiveness of our approach in different scenarios.
\begin{table}[!hbt]
	\centering
	\begin{adjustbox}{
			width=\linewidth
		}
		\begin{tabular}{@{}lrr|rrr@{}}
			\toprule
			\textbf{Dataset} & \textbf{1} & \textbf{0} & \textbf{Train} & \textbf{Val} & \textbf{Test} \\
			\midrule
            Alberta Flood (AF) & 5189 & 4842 & 6019 & 1002 & 3008 \\
            Boston Bombing (BB) & 5648 & 4364 & 6007 & 1001 & 3004 \\
            Kaggle (KL) & 3356 & 4336 & 4622 &  816 & 2164 \\
            Nepal Earthquake (NE) & 5527 & 6141 & 7000 & 1166 & 3502 \\
            Oklahoma Tornado (OT) & 4827 & 5165 & 5994 &  998 & 2997 \\
            Queensland  Floods (QF) & 5414 & 4619 & 6019 & 1003 & 3011 \\
            Sandy Hurricane (SH) & 6138 & 3870 & 6005 & 1001 & 3002 \\
            West Texas Explosion (WE) & 5246 & 4760 & 6001 & 1001 & 3004 \\ \bottomrule
		\end{tabular}
	\end{adjustbox}
	\caption{Details of $8$ disaster datasets used in our experiments. Classes $1$/$0$ indicates tweets relevant/irrelevant to the disaster.}
	\label{table:datasets}
\end{table}
\section{Results} \label{section:results} We now discuss the experimental results with reference to several important research questions to understand the usefulness of the proposed pretraining approach that we refer to as Graph Contrastive pretraining for Text (GCPT). 
We run our experiments on each dataset and report the scores averaged over all datasets in a collection. Refer to Appendix \S \ref{appendix} for more detailed dataset-wise results.

\noindent \textbf{Q1) Does GCPT outperform unsupervised pretrained embedding baselines?} \par
GCPT uses supervised data and propagates that information to individual tokens by graph contrastive learning. 
In this experiment, we want to see if our technique is useful over unsupervised pretraining, as done in Glove, Word2Vec, etc.
To that effect, we perform text classification on current data using a classifier with CE ($\mathbf{f_{CE}(\mathbf{O}_\texttt{TR})}$) initialized with different word embeddings to see which initialization leads to better performance. Specifically,
\begin{enumerate}
	\item \emph{Word embeddings pretrained on large corpus}: We use GloVe~\citep{glove} embedding pretrained on large general purpose corpora. In addition, we also use a Word2Vec model pretrained on a large corpus of disaster-related data with $9.4$ million tokens \citep{cnlp} denoted as CNLP. As CNLP is trained on disaster data only, we do not perform experiments on sentiment datasets using CNLP.
	\item \emph{Word embeddings trained on available corpus}: We train a Word2Vec \citep{word2vec} model on all data available from related datasets. We denote this embedding set as W2V.
\end{enumerate}
Table \ref{tab:embs_transfer} present our findings for these experiments on both disaster and sentiment datasets. GCPT outperforms all unsupervised pretrained embeddings due to its capability to capture and transfer supervised information effectively. We are able to achieve an absolute improvement of $3.3\%$ for disaster and $1.7\%$ for sentiment datasets. We also observe that GloVe outperforms both CNLP and W2V embeddings for majority of the datasets. For this reason, we use GloVe for the following experiments.
\begin{table}[!htbp]
	\centering
	\begin{tabular}{@{}lrr@{}}
		\toprule
		\textbf{Embedding} & \textbf{Disaster} & \textbf{Sentiment} \\
		\midrule
		GloVe   & 76.52 & \underline{56.11} \\
		W2V     & 76.28 & 55.23 \\
		CNLP    & \underline{76.73} & -     \\
		GCPT    & \textbf{80.03} & \textbf{57.81} \\
		\bottomrule
	\end{tabular}
	\caption{Performance comparison of GCPT with unsupervised pretrained embeddings aggregated over disaster and sentiment datasets.}
	\label{tab:embs_transfer}
\end{table}

\noindent \textbf{Q2) Is GCPT suitable for different types of neural networks?} \par
We verify the applicability of GCPT with different types of neural networks by replacing the BiLSTM classifier with other neural architectures such as MLP and CNN. For MLP, we represent an input example as the average of the embeddings of its tokens, and denote the method as MLP-BoW. We also use DenseCNN \citep{densecnn} and XML-CNN \citep{xmlcnn} following \citep{disaster_cnn} for disaster and DPCNN \citep{dpcnn} for sentiment datasets. The results are presented in Table \ref{tab:models_transfer}.
Our experiments show that GCPT performs well for almost all types of neural networks except MLP-BoW for disaster data, highlighting the benefit brought in by the refinement in the embedding space through the contrastive learning. In case of MLP-BoW, averaging over contrasted tokens within an example makes it difficult to capture proper context resulting in lower performance.
\begin{table}[!htbp]
	\centering
	\begin{adjustbox}{
		}
		\begin{tabular}{@{}lrr|rr@{}}
			\toprule
		& \multicolumn{2}{c|}{\textbf{Disaster}} & \multicolumn{2}{c}{\textbf{Sentiment}} \\
		\cmidrule(l){2-3} \cmidrule(l){4-5}
		\textbf{Model} & \textbf{CNLP} & \textbf{GCPT} & \textbf{GloVe} & \textbf{GCPT} \\
			\midrule
			MLP-BoW     & \textbf{53.74} & 50.95 & 52.94 & \textbf{55.48} \\
			CNN         & 77.30 & \textbf{78.17} & 54.89 & \textbf{56.85} \\
			DenseCNN    & 78.81 & \textbf{79.99} & -     & -     \\
			XML-CNN     & 71.74 & \textbf{73.90} & -     & -     \\
			DPCNN       & -     & -     & 55.36 & \textbf{57.38} \\
			\bottomrule
		\end{tabular}
	\end{adjustbox}
	\caption{GloVe vs GCPT performance with various neural networks.}
	\label{tab:models_transfer}
\end{table}

\noindent \textbf{Q3) Is token-level contrastive formulation effective?} \par
This experiment verifies the effectiveness of GCPT over other straw-man approaches of incorporating the supervised information by replacing the GCPT pretraining phase (i.e. replacing $\mathbf{h_{CL}^{T}(\mathbf{P})}$). We consider two approaches that utilize supervised information present in related data:
\begin{enumerate}
	\item \label{ecl} \emph{Example-level Contrastive Learning (ECL)}: In this intuitive approach, we employ contrastive learning over \textit{examples} instead of tokens. We pretrain a model with contrastive loss over examples from related datasets (i.e. $\mathbf{h_{CL}^{S}(\mathbf{P})}$ instead of $\mathbf{h_{CL}^{T}(\mathbf{P})}$) and finally re-train the classifier using the current dataset to fine-tune ($\mathbf{f_{CE}(\mathbf{O}_\texttt{TR})}$). We consider each example in the training set as a candidate item for pretraining. For each candidate, we randomly select $5$ examples from the same class as positive for that instance and $5$ examples from other classes as negative examples to contrast. Trained model weights and token embeddings are transferred from pretraining to fine-tuning dataset.
	\item \label{2f} \emph{Related Dataset Fine-tuning (RDF)}: Here, we treat all labeled data (both related and current dataset) as training data. We directly fine-tune a model using related datasets with CE loss instead of any fine-tuning (i.e. $\mathbf{f_{CE}(\mathbf{P})}$ instead of $\mathbf{h_{CL}^{T}(\mathbf{P})}$), followed by fine-tuning the model again with the current dataset ($\mathbf{f_{CE}(\mathbf{O}_\texttt{TR})}$). We compare with this baseline by introducing $\mathbf{h_{CL}^{T}(\mathbf{P})}$ to the front of the pipeline and replacing the embeddings in subsequent phases.
\end{enumerate}
Table \ref{tab:examcon} presents the results for ECL (example-level) compared with GCPT (token-level). Our findings show that GCPT outperforms ECL by $1.91\%$ and $1.69\%$ for disaster and sentiment datasets on average. However, ECL performs better for the KL dataset. As KL contains data from multiple disasters, token-level contrast might reduce separability among various disasters, confusing the model. From a practical standpoint, it is generally unexpected that data will be coming from multiple domains during an ongoing disaster.
\begin{table}[!htbp]
	\centering
	\begin{adjustbox}{
			width=0.99\linewidth
		}
		\begin{tabular}{@{}lrrr@{}}
			\toprule
			\textbf{Dataset} & \textbf{ECL} & \textbf{GCPT} \\ \midrule
			AF & 78.90 & \textbf{80.23} \\
			BB & 66.64 & \textbf{71.35} \\
			KL & \textbf{56.45} & 55.44 \\
			NE & 52.46 & \textbf{53.76} \\
			OT & 65.14 & \textbf{66.24} \\
			QF & 69.98 & \textbf{71.66} \\
			SH & 61.64 & \textbf{63.86} \\
			WE & 70.19 & \textbf{74.14} \\
			\midrule
			Avg & 65.17 & \textbf{67.08} \\
			\bottomrule
		\end{tabular}
		\quad
		\centering
		\begin{tabular}{@{}lrrr@{}}
			\toprule
			\textbf{Dataset} & \textbf{ECL} & \textbf{GCPT} \\ \midrule
			books       & 56.45 & \textbf{58.48} \\
			dvd         & 52.72 & \textbf{53.62} \\
			electronics & 61.77 & \textbf{64.66} \\
			kitchen     & 64.98 & \textbf{65.33} \\
			video       & 58.63 & \textbf{60.92} \\
			\midrule
			Avg         & 58.91 & \textbf{60.60} \\
			\bottomrule
		\end{tabular}
	\end{adjustbox}
	\caption{Performance comparison of ECL (Ref. \ref{ecl}) \& GCPT for disaster (left) sentiment (right) datasets.}
	\label{tab:examcon}
\end{table}
%

Recall that we achieved $80.03\%~\&~57.81\%$ (Ref. Table \ref{tab:embs_transfer}) for disaster and sentiment datasets when only current dataset data is utilized. Utilizing all available data to train the model (RDF setting), improves performance by $5.9\%~\&~0.46\%$ on average for disaster and sentiment datasets as shown in Table~\ref{tab:embs_rdf}. However, when we incorporate GCPT into the RDF setting, by replacing the token embeddings, we achieve even higher performance improvement of $7.81\%~\&~2.76\%$ for disaster and sentiment datasets respectively. We also perform the same experiment with other models (MLP-BoW, CNN, DenseCNN and XML-CNN) and found incorporating GCPT helps achieve better performance in $6$ out of $7$ scenarios as shown in Table \ref{tab:models_rdf}.
\begin{table}[!htbp]
	\centering
	\begin{tabular}{@{}lrr@{}}
		\toprule
		\textbf{Embedding} & \textbf{Disaster} & \textbf{Sentiment} \\
		\midrule
		GloVe   & \underline{82.62} & \underline{56.47} \\
		W2V     & 82.01 & 55.79 \\
		CNLP    & 82.60 & -     \\
		\midrule
		RDF (GCPT)& \textbf{84.32} & \textbf{58.43} \\
		\bottomrule
	\end{tabular}
	\caption{Performance comparison of GCPT with unsupervised pretrained embeddings in RDF scenario aggregated over disaster and sentiment datasets.}
	\label{tab:embs_rdf}
\end{table}
\begin{table}[!htbp]
\centering
\begin{adjustbox}{
	}
	\begin{tabular}{@{}lrr|rr@{}}
		\toprule
		& \multicolumn{2}{c|}{\textbf{Disaster}} & \multicolumn{2}{c}{\textbf{Sentiment}} \\
		\cmidrule(l){2-3} \cmidrule(l){4-5}
		\textbf{Model} & \textbf{GloVe} & \textbf{GCPT} & \textbf{GloVe} & \textbf{GCPT} \\
		\midrule
		MLP-BoW     & 65.98 & \textbf{70.02} & \textbf{50.46} & 50.22 \\
		CNN         & 81.94 & \textbf{83.25} & 51.84 & \textbf{53.37} \\
		DenseCNN    & 83.59 & \textbf{84.00} & -     & -     \\
		XML-CNN     & 83.27 & \textbf{84.02} & -     & -     \\
		DPCNN       & -     & -     & 58.36 & \textbf{62.60} \\
		\bottomrule
	\end{tabular}
\end{adjustbox}
\caption{GloVe and GCPT performance with various neural models in RDF scenario.}
\label{tab:models_rdf}
\end{table}

\noindent \textbf{Q4) What if no labeled data is available for the current dataset?} \par
In domains like disaster, it is expected that no labeled data is available for an ongoing disaster. Data labelling during an ongoing disaster is very costly. In such a scenario, we can deploy models trained on labeled data from related disasters only (i.e. $\mathbf{f_{C}(\mathbf{P})}$). We experiment with such a scenario by learning token representations and refer to this scenario as Zero-Shot setting.

Table \ref{tab:embs_zeroshot} reports the performance among different pretraining techniques. We observe that GCPT outperforms other baselines for all datasets except NE. This is interesting, as NE contains tweets from earthquake which is different from rest of the disaster types. However, We observe CNLP performs better than GCPT as it is trained on a large corpora of disaster data  which includes earthquake disaster. Table \ref{tab:models_zeroshot} presents performance comparison of W2V and GCPT for different models. We use W2V for this experiment, as it performs best (Ref. Table \ref{tab:embs_zeroshot}) in Zero-Shot scenario. We find a similar observation that GCPT outperforms W2V in $6$ out of $7$ scenarios.
We are able to improve performance $5.35\%$ on average just by replacing the token embeddings in zero-shot scenario. This is a very important in case of a new disaster as usually no labeled data is available in such a scenario. Our experiment show that we can utilize labeled data from related disasters effectively in such a scenario, saving crucial time.
\begin{table}[!htbp]
	\centering
	\begin{tabular}{@{}lrr@{}}
		\toprule
		\textbf{Embedding} & \textbf{Disaster} & \textbf{Sentiment} \\
		\midrule
		GloVe   & 62.76 & 52.22 \\
		W2V     & \underline{67.38} & \underline{53.16} \\
		CNLP    & 64.13 & -     \\
		GCPT    & \textbf{72.73} & \textbf{55.64} \\
		\bottomrule
	\end{tabular}
	\caption{Performance comparison of GCPT with unsupervised pretrained embeddings in Zero-Shot scenario aggregated over disaster and sentiment datasets.}
	\label{tab:embs_zeroshot}
\end{table}
\begin{table}[!htbp]
\centering
\begin{adjustbox}{
	}
	\begin{tabular}{@{}lrr|rr@{}}
		\toprule
		& \multicolumn{2}{c|}{\textbf{Disaster}} & \multicolumn{2}{c}{\textbf{Sentiment}} \\
		\cmidrule(l){2-3} \cmidrule(l){4-5}
		\textbf{Model} & \textbf{W2V} & \textbf{GCPT} & \textbf{W2V} & \textbf{GCPT} \\
		\midrule
		MLP-BoW     & 51.10 & \textbf{57.15} & \textbf{50.89} & 50.46 \\
		CNN         & 61.27 & \textbf{67.20} & 51.44 & \textbf{54.45} \\
		DenseCNN    & 61.41 & \textbf{64.34} & -     & -     \\
		XML-CNN     & 59.53 & \textbf{62.81} & -     & -     \\
		DPCNN       & -     & -     & 54.33 & \textbf{56.45} \\
		\bottomrule
	\end{tabular}
\end{adjustbox}
\caption{W2V and GCPT performance with various neural models in Zero-Shot scenario.}
\label{tab:models_zeroshot}
\end{table}

\noindent \textbf{Q5) Is GCPT effective to improve performance of transformer models?} \par
Transformer based language models require large amount of end-task specific data to perform well. Due to this, we wanted to experiment with transformer-based models by applying GCPT in knowledge distillation setting. We follow a similar setting as \cite{kd_paper} to fine-tune BERT \citep{bert} using Hard Knowledge Distillation. To analyze the effect of availability of additional data, we divide the current dataset train set into two equal splits and treat one as labeled and the other as unlabeled. As a baseline, a BERT model with a MLP classifier head (Bert) is trained using the labeled split. We train a 2-layer BiLSTM model with GCPT embeddings with the labeled split as teacher and predict on the unlabeled split of the data. Finally, we train another BERT with a MLP classifier (BKD) using labeled part with original labels and unlabeled part with labels predicted by the teacher model. We observe in Table \ref{tab:kd} that additional data with GCPT predicted labels boost the model performance significantly for all $13$ datasets.
%
%
\begin{table}[!htbp]
\centering
\begin{adjustbox}{
		width=0.99\linewidth
	}
	\begin{tabular}{@{}lrr@{}}
		\toprule
		\textbf{Dataset} & \textbf{Bert} & \textbf{BKD} \\ \midrule
		AF & 87.14 & \textbf{93.74} \\
		BB & 87.01 & \textbf{91.11} \\
		KL & 72.44 & \textbf{77.36} \\
		NE & 70.62 & \textbf{75.79} \\
		OT & 88.37 & \textbf{91.72} \\
		QF & 90.85 & \textbf{94.68} \\
		SH & 85.23 & \textbf{88.35} \\
		WE & 91.82 & \textbf{96.70} \\
		\midrule
		Avg & 84.18 & \textbf{88.68} \\
		\bottomrule
	\end{tabular}
	\quad
	\centering
	\begin{tabular}{@{}lrrr@{}}
		\toprule
		\textbf{Dataset} & \textbf{Bert} & \textbf{BKD} \\ \midrule
		books       & 86.66 & \textbf{89.74} \\
		dvd         & 87.41 & \textbf{91.00} \\
		electronics & 87.01 & \textbf{92.25} \\
		kitchen     & 89.21 & \textbf{94.90} \\
		video       & 87.64 & \textbf{93.24} \\
		\midrule
		Avg & 87.58 & \textbf{92.22} \\
		\bottomrule
	\end{tabular}
\end{adjustbox}
\caption{Performance comparison of Bert \& BKD.}
\label{tab:kd}
\end{table}
\section{Conclusion \& Future Work} \label{conclusion} In this paper, we propose a supervised contrastive pretraining technique to effectively utilize labeled data from related datasets. Our proposed supervised graph contrastive pretraining brings tokens with similar context from same class closer. By contrasting non-contextual tokens, the learned embeddings can capture the supervised information in meaningful way that generalize well to follow-up datasets. Our experiments over $8$ disaster and $5$ sentiment datasets show that this separability helps achieve better performance compared to utilizing the supervised information in other ways. We also show that Transformer based models can be taught (using knowledge distillation) to utilize this information.
In future, we plan to utilize sub-word level token information for the graph construction which we believe will reduce the vocabulary mismatch and further improve performance.
\nocite{casella2021statistical}

\nocite{casella2021statistical}
\bibliographystyle{ACM-Reference-Format}
\bibliography{sample-sigconf}

\appendix
\section{Technical Appendix} \label{appendix}
\subsection{Hyperparameter Configuration} \label{hyperparam}
We have two sets of hyperparameters to tune, during pre-training and fine-tuning. We set minimum token frequency as $5$ and minimum class conditional probability (i.e. $\tau$) as $0.9$. Maximum number of epochs during pre-training was set to $80$ with learning rate $0.005$. These values were set based on classification performance in the development set. We used the above pre-training hyper-parameters throughout all experiments.
Fine-tuning hyperparameters were also tuned based on the performance on the validation set with patience $4$. Learning rate was chosen from values $(10^{-2}, 10^{-3}, 10^{-4})$ with maximum epoch $10$ for disaster and $30$ for sentiment datasets. $300$ dimensional input vectors were used for all experiments. We set hidden dimension as $300$ whenever required for any experiments.
A system with a `Nvidia Tesla P100' GPU with $64$GB RAM and $56$ core CPU was used for all experiments.
For BERT related experiments, maximum number of epochs were set to $5$ with sequence length $128$.

\subsection{Vocabulary Overlaps} \label{vocab_overlap}
We mention the token-wise overlap percentage between related and current data for each dataset after pre-processing with minimum token frequency of $5$.
\begin{table}[!htbp]
\centering
    \begin{tabular}{@{}lr@{}}
    \toprule
        \textbf{Dataset} & \textbf{Overlap \%} \\ \midrule
        AF & 81.63 \\
        BB & 89.42 \\
        KL & 65.46 \\
        NE & 73.17 \\
        OT & 90.02 \\
        QF & 81.97 \\
        SH & 89.37 \\
        WE & 87.28 \\
    \bottomrule
    \end{tabular}
    \quad
    \centering
    \begin{tabular}{@{}lr@{}}
    \toprule
        \textbf{Dataset} & \textbf{Overlap \%} \\ \midrule
        books       & 80.94 \\
        dvd         & 87.44 \\
        electronics & 83.41 \\
        kitchen     & 84.36 \\
        video       & 88.43 \\
    \bottomrule
    \end{tabular}
\caption{Details of vocabulary overlap between related and current data for disaster (left) and sentiment (right) datasets.}
\label{tab:vocab_overlap}
\end{table}

\subsection{Dataset-wise Results} \label{granular_results}
Here, we report all dataset-wise results.
\begin{table}[!htbp]
    \centering
    \begin{tabular}{@{}lrrrrr@{}}
    \toprule
    \textbf{Dataset} & \textbf{GloVe} & \textbf{CNLP} & \textbf{W2V} & \textbf{GCPT} \\ 
        \midrule
        AF & 75.92 & 77.36 & \underline{79.74} & \textbf{83.87} \\
        BB & \underline{79.96} & 78.72 & 78.39 & \textbf{80.84} \\
        KL & 54.84 & 55.51 & \underline{57.53} & \textbf{61.30} \\
        NE & \underline{64.71} & 64.16 & 62.93 & \textbf{65.84} \\
        OT & \underline{80.56} & 80.09 & 77.42 & \textbf{82.49} \\
        QF & 81.54 & \underline{84.62} & 84.28 & \textbf{89.61} \\
        SH & \underline{85.29} & 85.10 & 84.75 & \textbf{85.61} \\
        WE & \underline{89.38} & 88.31 & 85.24 & \textbf{90.75} \\ 
        \bottomrule
    \end{tabular}
\caption{Performance comparison of GCPT with unsupervised pre-trained embeddings for disaster datasets.}
\label{tab:embs_disaster_transfer_granular}
\end{table}
\begin{table}[!htbp]
    \centering
    \begin{tabular}{@{}lrrrr@{}}
    \toprule
      \textbf{Dataset} & \textbf{GloVe} & \textbf{W2V} & \textbf{GCPT} \\
        \midrule
        books       & \underline{51.60} & 50.17 & \textbf{52.88} \\
        dvd         & \underline{55.82} & 55.11 & \textbf{56.43} \\
        electronics & \underline{58.24} & 56.36 & \textbf{60.79} \\
        kitchen     & \underline{60.72} & 59.63 & \textbf{63.44} \\
        video       & 54.19 & \underline{54.89} & \textbf{55.51} \\
    \bottomrule
    \end{tabular}
\caption{Performance comparison of GCPT with unsupervised pre-trained embeddings for sentiment datasets.}
\label{tab:embs_sentiment_transfer_granular}
\end{table}

\begin{table}[!htbp]
\centering
\scalebox{0.8}{
\begin{tabular}{@{}lrrrr@{}}
\toprule
\textbf{Dataset} & \textbf{MLP-BoW} & \textbf{CNN} & \textbf{DenseCNN} & \textbf{XML-CNN} \\
\midrule
    AF & 56.38 / 55.22 & 79.17 / 81.42 & 80.05 / 82.78 & 68.73 / 69.92 \\
    BB & 53.30 / 50.80 & 78.92 / 79.62 & 80.53 / 83.30 & 71.30 / 72.81 \\
    KL & 44.71 / 42.28 & 54.02 / 57.83 & 54.70 / 55.61 & 53.42 / 54.99 \\
    NE & 48.24 / 44.23 & 63.37 / 64.71 & 63.76 / 64.47 & 62.65 / 63.62 \\
    OT & 58.84 / 52.83 & 80.88 / 78.87 & 84.28 / 85.06 & 78.94 / 79.08 \\
    QF & 58.08 / 56.99 & 88.77 / 87.59 & 88.77 / 89.46 & 81.66 / 85.36 \\
    SH & 57.21 / 58.27 & 87.45 / 87.75 & 86.99 / 87.52 & 79.59 / 82.02 \\
    WE & 53.21 / 47.04 & 85.84 / 87.62 & 91.81 / 91.75 & 77.64 / 83.46 \\
\bottomrule
\end{tabular}
}
\caption{CNLP vs GCPT performance with various neural models for disaster datasets.}
\label{tab:models_disaster_transfer_granular}
\end{table}
%
\begin{table}[!htbp]
\centering
\scalebox{0.8}{
\begin{tabular}{@{}lrrrr@{}}
\toprule
\textbf{Dataset} & \textbf{MLP-BoW} & \textbf{CNN} & \textbf{DenseCNN} & \textbf{XML-CNN} \\
\midrule
    AF & 76.84 / 77.20  &  87.44 / 88.10  &  88.51 / 91.87  & 87.96 / 88.43 \\
    BB & 66.09 / 69.11  &  86.03 / 87.60  &  86.84 / 86.81  & 84.62 / 86.06 \\
    KL & 55.69 / 57.47  &  60.76 / 62.67  &  60.74 / 59.11  & 61.74 / 61.82 \\
    NE & 50.75 / 51.44  &  64.86 / 65.22  &  67.23 / 65.48  & 67.77 / 69.07 \\
    OT & 62.48 / 66.94  &  83.24 / 85.42  &  90.69 / 91.36  & 88.56 / 89.16 \\
    QF & 75.53 / 79.50  &  92.03 / 94.96  &  92.49 / 93.46  & 94.02 / 94.86 \\
    SH & 68.08 / 77.02  &  88.70 / 89.29  &  89.75 / 90.34  & 89.01 / 89.51 \\
    WE & 72.45 / 81.50  &  92.48 / 92.81  &  92.48 / 93.61  & 92.48 / 93.31 \\
\bottomrule
\end{tabular}
}
\caption{GloVe and GCPT performance with various neural models for disaster datasets in RDF scenario.}
\label{tab:models_disaster_alltrain_granular}
\end{table}
\begin{table}[!htbp]
\centering
\scalebox{0.8}{
\begin{tabular}{@{}lrrrr@{}}
\toprule
\textbf{Dataset} & \textbf{MLP-BoW} & \textbf{CNN} & \textbf{DenseCNN} & \textbf{XML-CNN} \\
\midrule
    AF & 57.69 / 59.66 & 74.66 / 76.41 & 74.33 / 76.15 & 71.72 / 74.54 \\
    BB & 58.71 / 67.76 & 62.30 / 68.82 & 61.25 / 64.31 & 59.74 / 60.37 \\
    KL & 42.35 / 55.26 & 54.79 / 54.97 & 56.08 / 55.86 & 55.37 / 51.71 \\
    NE & 37.14 / 32.52 & 51.06 / 52.53 & 51.54 / 52.77 & 47.94 / 50.17 \\
    OT & 44.51 / 62.57 & 61.08 / 65.39 & 63.22 / 65.80 & 57.57 / 64.90 \\
    QF & 58.21 / 50.95 & 66.11 / 81.25 & 66.96 / 71.26 & 68.85 / 77.63 \\
    SH & 57.12 / 70.82 & 57.88 / 64.53 & 54.84 / 58.33 & 55.37 / 57.79 \\
    WE & 53.07 / 57.68 & 62.32 / 73.72 & 63.11 / 70.24 & 59.75 / 66.08 \\
\bottomrule
\end{tabular}
}
\caption{W2V and GCPT performance with various neural models for disaster datasets in Zero-Shot scenario.}
\label{tab:models_disaster_zeroshot_granular}
\end{table}
\begin{table}[!htbp]
\centering
\begin{tabular}{@{}lrrr@{}}
\toprule
\textbf{Dataset} & \textbf{MLP-BoW} & \textbf{CNN} & \textbf{DPCNN} \\
    \midrule
    books       & 52.55 / 55.82 & 50.96 / 55.01 & 54.47 / 56.21 \\
    dvd         & 55.31 / 57.60 & 52.43 / 54.71 & 51.19 / 53.75 \\
    electronics & 50.37 / 57.39 & 60.44 / 61.89 & 60.25 / 60.72 \\
    kitchen     & 55.13 / 55.09 & 59.08 / 57.01 & 58.40 / 61.02 \\
    video       & 51.37 / 51.54 & 51.58 / 55.66 & 52.49 / 55.21 \\
\bottomrule
\end{tabular}
\caption{GloVe vs GCPT performance with various neural models for sentiment datasets.}
\label{tab:models_sentiment_transfer_granular}
\end{table}%
\begin{table}[!htbp]
\centering
\begin{tabular}{@{}lrrrrr@{}}
\toprule
\textbf{Dataset} & \textbf{GloVe} & \textbf{CNLP} & \textbf{W2V} & \textbf{GCPT} \\ \midrule
    AF & \underline{87.47} & 87.38 & 86.85 & \textbf{89.31} \\
    BB & 85.31 & \underline{85.60} & 83.24 & \textbf{87.18} \\
    KL & \underline{62.32} & 62.07 & 60.53 & \textbf{63.44} \\
    NE & \underline{67.99} & 67.61 & 65.81 & \textbf{68.55} \\
    OT & \underline{87.31} & 87.13 & 85.51 & \textbf{88.86} \\
    QF & 89.52 & 89.69 & \underline{90.49} & \textbf{93.32} \\
    SH & 88.23 & \underline{88.62} & 88.38 & \textbf{89.38} \\
    WE & \underline{92.88} & 92.72 & 92.70 & \textbf{94.59} \\
\bottomrule
\end{tabular}
\caption{Performance comparison of GCPT with unsupervised pre-trained embeddings for disaster datasets in RDF scenario.}
\label{tab:embs_disaster_alltrain_granular}
\end{table}
\begin{table}[!htbp]
\centering
\begin{tabular}{@{}lrrrr@{}}
\toprule
  \textbf{Dataset} & \textbf{GloVe} & \textbf{W2V} & \textbf{GCPT} \\
    \midrule
        books       & \underline{52.69} & 51.73 & \textbf{58.16} \\
        dvd         & \underline{56.22} & 56.16 & \textbf{57.74} \\
        electronics & 54.96 & \underline{55.74} & \textbf{59.21} \\
        kitchen     & \underline{57.24} & 56.99 & \textbf{58.80} \\
        video       & \underline{56.27} & 55.87 & \textbf{58.25} \\
\bottomrule
\end{tabular}
\caption{Performance comparison of GCPT with unsupervised pre-trained embeddings for sentiment datasets in RDF scenario.}
\label{tab:embs_sentiment_alltrain_granular}
\end{table}
\begin{table}[!htbp]
\centering
\begin{tabular}{@{}lrrr@{}}
\toprule
\textbf{Dataset} & \textbf{MLP-BoW} & \textbf{CNN} & \textbf{DPCNN} \\
    \midrule
    books       & 45.92 / 45.26 & 49.75 / 51.93 & 55.87 / 59.16 \\
    dvd         & 49.22 / 48.68 & 50.44 / 50.57 & 60.50 / 61.46 \\
    electronics & 52.77 / 52.89 & 51.46 / 55.73 & 55.99 / 62.56 \\
    kitchen     & 52.49 / 52.43 & 55.43 / 56.75 & 60.25 / 65.58 \\
    video       & 51.92 / 51.87 & 52.14 / 51.89 & 59.19 / 64.24 \\
\bottomrule
\end{tabular}
\caption{GloVe and GCPT performance with various neural models for sentiment datasets in RDF scenario.}
\label{tab:models_sentiment_alltrain_granular}
\end{table}
\begin{table}[!htbp]
\centering
\begin{tabular}{@{}lrrrrr@{}}
\toprule
\textbf{Dataset} & \textbf{Glove} & \textbf{CNLP} & \textbf{W2V} & \textbf{GCPT} \\ \midrule
    AF & 65.96 & \underline{66.64} & 66.42 & \textbf{72.82} \\
    BB & 64.89 & 65.85 & \underline{73.52} & \textbf{76.60} \\
    KL & 52.24 & 53.74 & \underline{55.69} & \textbf{57.19} \\
    NE & 52.43 & \textbf{53.76} & 52.96 & \underline{53.28} \\
    OT & 68.34 & 64.50 & \underline{70.70} & \textbf{80.91} \\
    QF & 67.81 & 73.00 & \underline{76.66} & \textbf{81.84} \\
    SH & 61.15 & 63.70 & \underline{66.56} & \textbf{73.05} \\
    WE & 69.26 & 71.86 & \underline{76.54} & \textbf{85.49} \\
\bottomrule
\end{tabular}
\caption{Performance comparison of GCPT with unsupervised pre-trained embeddings for disaster datasets in Zero-Shot scenario.}
\label{tab:embs_disaster_zeroshot_granular}
\end{table}
\begin{table}[!htbp]
\centering
\begin{tabular}{@{}lrrrr@{}}
\toprule
  \textbf{Dataset} & \textbf{GloVe} & \textbf{W2V} & \textbf{GCPT} \\
    \midrule
    books       & \underline{51.37} & 50.94 & \textbf{54.53} \\
    dvd         & 51.75 & \underline{53.10} & \textbf{54.43} \\
    electronics & 53.27 & \underline{54.96} & \textbf{58.52} \\
    kitchen     & 51.22 & \underline{52.59} & \textbf{54.25} \\
    video       & 53.49 & \underline{54.23} & \textbf{56.51} \\
\bottomrule
\end{tabular}
\caption{Performance comparison of GCPT with unsupervised pre-trained embeddings for sentiment datasets in Zero-Shot scenario.}
\label{tab:embs_sentiment_zeroshot_granular}
\end{table}
\begin{table}[!htbp]
\centering
\begin{tabular}{@{}lrrr@{}}
\toprule
\textbf{Dataset} & \textbf{MLP-BoW} & \textbf{CNN} & \textbf{DPCNN} \\
    \midrule
    books       & 48.17 / 50.43 & 48.90 / 54.26 & 55.58 / 57.89 \\
    dvd         & 50.50 / 48.52 & 51.46 / 53.01 & 52.43 / 55.15 \\
    electronics & 52.29 / 55.18 & 52.50 / 56.52 & 53.15 / 57.68 \\
    kitchen     & 51.95 / 52.43 & 50.61 / 50.82 & 53.72 / 53.81 \\
    video       & 51.58 / 45.77 & 53.77 / 57.65 & 56.81 / 57.72 \\
\bottomrule
\end{tabular}
\caption{W2V and GCPT performance with various neural models for sentiment datasets in Zero-Shot scenario.}
\label{tab:models_sentiment_zeroshot_granular}
\end{table}
\end{document}